\documentclass{article}
\usepackage{graphicx}
\usepackage{ifthen}
\usepackage[final]{neurips_2025} 

\newcommand{\apptabfontsize}{\scriptsize} 

\usepackage[utf8]{inputenc}
\usepackage[T1]{fontenc}
\usepackage{microtype}

\usepackage{amsmath, amssymb, amsthm}
\usepackage{mathtools}
\usepackage{booktabs}
\usepackage{graphicx}
\usepackage{algorithm}
\usepackage{algorithmic}
\usepackage{hyperref}
\usepackage{tikz}
\usepackage{natbib}
\usepackage{adjustbox}
\usepackage{graphicx} 
\usepackage{float}    

\newtheorem{theorem}{Theorem}
\newtheorem{proposition}{Proposition}

\newtheorem{corollary}{Corollary}
\newtheorem{assumption}{Assumption}
\newtheorem{remark}{Remark}

\newcommand{\pct}[1]{#1\%}

\usepackage{xcolor}

\usetikzlibrary{calc,fit,backgrounds,matrix,positioning}

\definecolor{oiBlue}{HTML}{0072B2}
\definecolor{oiGreen}{HTML}{009E73}
\definecolor{oiOrange}{HTML}{E69F00}

\colorlet{laneScenario}{oiBlue!8}
\colorlet{laneSignal}{oiGreen!8}
\colorlet{laneAlloc}{oiOrange!10}

\usepackage{pgfplots}
\pgfplotsset{
  compat=1.18,
  every axis/.append style={
    tick label style={font=\scriptsize},
    label style={font=\scriptsize},
    title style={font=\scriptsize, yshift=-1.2mm},
    legend style={font=\scriptsize, cells={anchor=west}}
  }
}
\usepgfplotslibrary{groupplots,statistics,colorbrewer}

\usepackage{subcaption}

\usetikzlibrary{arrows.meta,positioning,calc,fit,shapes.geometric,shapes.misc,backgrounds,matrix}

\title{Multi-Agent Regime-Conditioned Diffusion (MARCD) for CVaR-Constrained Portfolio Decisions}

\author{%
  \textbf{Ali Atiah Alzahrani}$^{\dagger}$
}

\usepackage{float}
\makeatletter
\def\@neuripsyear{2025}
\def\@neuripsordinal{39th}
\providecommand{\@trackname}{}
\begin{document}
\makeatletter
\maketitle

\begingroup
\renewcommand\thefootnote{}\footnotetext{%
\noindent\begin{minipage}{0.96\linewidth}
\raggedright
$^{\dagger}$\textbf{Corresponding author:} \texttt{alialzahrani@pif.gov.sa}\\[3pt]
$^{\dagger}$The views expressed are those of the author and do not reflect the views of any other individual or entity. This material is for research purposes only and does not constitute investment advice.
\end{minipage}
}
\addtocounter{footnote}{-1}
\endgroup

\begin{abstract}
We examine whether \emph{regime-conditioned} generative scenarios combined with a convex CVaR allocator improve portfolio decisions under regime shifts. We present \textbf{MARCD}, a generative-to-decision framework with: (i) a Gaussian HMM to infer latent regimes; (ii) a diffusion generator that produces \emph{regime-conditioned} scenarios; (iii) signal extraction via blended, shrunk moments; and (iv) a governed \emph{CVaR epigraph} QP. 
\textit{Contributions.} Within the Scenario stage we introduce a \emph{tail-weighted} diffusion objective that up-weights low-quantile outcomes relevant for drawdowns and a \emph{regime-expert (MoE) denoiser} whose gate increases with crisis posteriors; both are evaluated end-to-end through the allocator. 
Under strict walk-forward on liquid multi-asset ETFs (2005--2025), MARCD exhibits stronger scenario calibration and materially smaller drawdowns: \textbf{MaxDD 9.3\%} versus \textbf{14.1\%} for BL (a \textbf{34\%} reduction) over 2020--2025 OOS. The framework provides an auditable pipeline with explicit budget, box, and turnover constraints, demonstrating the value of decision-aware generative modeling in finance.
\end{abstract}

\begin{figure}[H]
  \centering
  \begin{adjustbox}{max width=\textwidth}
  \begin{tikzpicture}[
    line cap=round,
    >={Stealth[length=4pt,width=6pt]},
    box/.style={rounded corners=2pt, draw=black!80, line width=0.8pt, fill=white, align=center,
      minimum height=13mm, inner sep=3pt, font=\normalsize, text width=3.65cm},
    arrS/.style={->, line width=0.9pt, draw=oiBlue!85,  shorten <=3pt, shorten >=3pt, solid},
    arrG/.style={->, line width=0.9pt, draw=oiGreen!85, shorten <=3pt, shorten >=3pt, densely dashed},
    arrA/.style={->, line width=0.9pt, draw=oiOrange!90,shorten <=3pt, shorten >=3pt, dash pattern=on 2pt off 1.2pt},
    lab/.style={font=\footnotesize, text=black},
    row/.style={row sep=0mm, column sep=2mm} 
  ]

    \matrix (row) [row] {
      \node[box] (data) {Market data\\(prices, factors)}; &
      \node[box] (reg)  {Regime inference\\(HMM / filters)}; &
      \node[box] (gen)  {Regime-conditioned\\diffusion generator\\
        {\footnotesize Tail-weighted loss $(q,\eta)$; Regime-MoE (gate $g_t$)}}; &
      \node[draw=none, minimum width=9mm, inner sep=0pt] (gap1) {}; & 
      \node[box] (sig)  {Signal extraction\\(risk/alpha features)\\
        {\footnotesize Blend $\lambda$: $(\hat{\mu}_t,\hat{\Sigma}_t)$; shrinkage}}; &
      \node[draw=none, minimum width=9mm, inner sep=0pt] (gap2) {}; & 
      \node[box] (alloc){Allocation\\(CVaR-QP, $\alpha{=}0.95$)\\
        {\footnotesize box \& turnover cap $\tau$;\; optional $+\kappa\|\Delta w\|_1$}}; &
      \node[draw=none, minimum width=12mm, inner sep=0pt] (gap3) {}; & 
      \node[box, text width=3.05cm] (eval) {Backtest\\metrics}; \\ 
    };

    \draw[arrS] (data) -- (reg);
    \draw[arrS] (reg)  -- (gen);

    \draw[arrG] ([xshift=1mm]gen.east) -- ([xshift=-1mm]sig.west)
      node[pos=.5, above=2.2mm, lab]{\(N\) scenarios};

    \draw[arrA] ([xshift=1mm]sig.east) -- ([xshift=-1mm]alloc.west)
      node[pos=.5, above=2.2mm, lab]{\((\hat{\mu}_t,\hat{\Sigma}_t)\)};

    \draw[arrA] ([xshift=1mm]alloc.east) -- ([xshift=-2mm]eval.west)
      node[pos=.5, above=2.6mm, lab]{trades, P\&L};

    \draw[arrS] (reg.north) to[out=60,in=120]
      node[pos=.5, below=1.2mm, lab]{regime posteriors \(\boldsymbol{\pi}_t\)}
      (gen.north);

    \begin{scope}[on background layer]
      \node[rounded corners=3pt, draw=oiBlue!40,  fill=laneScenario, inner sep=4pt, fit=(data)(reg)(gen)] (lane1) {};
      \node[rounded corners=3pt, draw=oiGreen!40, fill=laneSignal,  inner sep=4pt, fit=(sig)]           (lane2) {};
      \node[rounded corners=3pt, draw=oiOrange!45,fill=laneAlloc,   inner sep=4pt, fit=(alloc)(eval)]   (lane3) {};
    \end{scope}

    \node[font=\bfseries\large, anchor=west, text=oiBlue!90]
      at ([xshift=2pt,yshift=3pt]lane1.north west) {Scenario};
    \node[font=\bfseries\large, anchor=west, text=oiGreen!90]
      at ([xshift=2pt,yshift=3pt]lane2.north west) {Signal};
    \node[font=\bfseries\large, anchor=west, text=oiOrange!90]
      at ([xshift=2pt,yshift=3pt]lane3.north west) {Allocation};

  \end{tikzpicture}
  \end{adjustbox}
  \caption{Multi-Agent Regime-Conditioned Diffusion (MARCD). }
  \label{fig:pipeline_swim_color}
\end{figure}
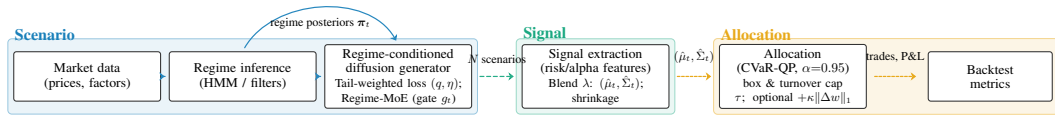

\section{Introduction}
Financial returns are non-stationary, with abrupt regime changes (e.g., 2008, 2020, 2022) that challenge mean--variance optimization (MVO) \citep{markowitz1952portfolio} and Black--Litterman (BL) \citep{black1992global}. Modern generative models---GANs \citep{yoon2019time}, diffusion \citep{rasul2021autoregressive}, transformers \citep{zhou2023informer}---produce realistic sequences but are often unconditioned and decoupled from decisions. Large peak-to-trough losses are driven by left-tail co-movements that standard diffusion training (MSE) tends to underweight, and a single denoiser blurred across regimes can dilute crisis dynamics.

We refer to our method as \emph{Multi-Agent Regime-Conditioned Diffusion (MARCD)} and use \textbf{MARCD} thereafter. MARCD aligns a diffusion generator to HMM posteriors while explicitly improving tail fidelity and crisis behavior: we employ a \emph{tail-weighted} diffusion objective to emphasize adverse outcomes that govern realized drawdowns, and we augment the denoiser with a lightweight \emph{regime expert} (mixture-of-experts) whose gate increases with the crisis posterior, enriching co-crash structure without degrading calm-regime fit. Generated scenarios feed a CVaR-focused, turnover-aware allocator within an auditable, walk-forward protocol.

\textbf{Contributions.} We integrate regime modeling, tail-aware generation, and robust allocation: (i) regime-conditioned diffusion aligned to HMM posteriors with a tail-weighted training objective that improves left-tail dependence; (ii) a regime-aware expert denoiser (MoE) that specializes to high-volatility states via learned gating; (iii) a multi-agent pipeline translating samples to decisions; and (iv) a CVaR-focused allocation with turnover control and explicit governance (walk-forward, constraints).

\section{Related Work}
\textbf{Regimes.} HMMs capture structural breaks and state dependence \citep{hamilton1989new,kim1994dynamic,ang2012regimes}; we use their posteriors as conditioning signals for both generation and allocation. \textbf{Generative time series.} GAN- and VAE-based generators (e.g., TimeGAN and TimeVAE) and diffusion models improve realism and calibration \citep{yoon2019time,desai2021timevae,rasul2021autoregressive,tashiro2021csdi}, yet are often unconditioned and loosely coupled to portfolio decisions; recent TS diffusions (\emph{TSDiff}, \emph{mr-Diff}) and a finance simulator (\emph{CTS-GAN}) are complementary to our regime-conditioned, decision-tied approach \citep{kollovieh2023tsdiff,shen2024mrdiff,istiaque2024ctsgan}. \textbf{Robust portfolios.} Distributionally robust and CVaR formulations address fat tails \citep{delage2010distributionally,rockafellar2000optimization} but hinge on scenario quality; we feed a regime-conditioned set into a convex allocator. \textbf{Multi-agent.} Building on coordination frameworks \citep{lowe2017multi}, we orchestrate Scenario$\to$Signal$\to$Allocation with light, role-based agents. \textbf{Sequence models.} LSTMs \citep{hochreiter1997lstm,fischer2018deep}, TCNs \citep{bai2018tcn}, and Transformers \citep{lim2021tft} remain standard forecasters; we treat point-forecast models as standard baselines, focusing instead on regime-conditioned \emph{distributional} generation wired to a convex CVaR allocator under strict walk-forward evaluation.

\section{Methodology}
\noindent\textit{Setup.} 
At rebalancing date $t$, let $\mathbf{R}_t\in\mathbb{R}^d$ denote the observed return vector and $\mathbf{w}_t\in\mathbb{R}^d$ the portfolio weights, which satisfy
\begin{equation}
\mathbf{1}^\top \mathbf{w}_t = 1,\qquad \boldsymbol{\ell}\le \mathbf{w}_t \le \boldsymbol{u}.
\end{equation}
Market regimes are modeled by a $K$-state Gaussian HMM with latent state $S_t$ and posterior vector $\boldsymbol{\pi}_t$, where $\pi_{t,k}=P(S_t{=}k\mid \mathbf{R}_{1:t})$. The generator produces $N$ next-period scenarios $\{\mathbf{r}^{(i)}_{t+1}\}_{i=1}^{N}$. Turnover is $\tau_t=\|\mathbf{w}_t-\mathbf{w}_{t-1}\|_1$. To avoid notation overload, we reserve $\alpha_s$ for the diffusion schedule and $\alpha$ for the allocator’s CVaR level, use $\lambda$ for scenario blending, and $\lambda_\mu$ for the expected-return weight in the allocator.

\subsection{Method Overview and Rationale}\label{sec:overview}
Our objective is to reduce pathwise drawdowns while preserving risk-adjusted return under strict walk-forward governance. We posit that returns exhibit regime-dependent higher-order dependence and that left-tail comovements matter more for realized drawdowns than central calibration. Accordingly, we bias generation toward adverse outcomes, make the generator responsive to regime signals, and align the allocator with tail-focused yet convex objectives that remain auditable and turnover-aware. Concretely, this yields a tail-weighted diffusion objective to improve co-crash fidelity, a regime-aware denoiser that specializes to high-volatility episodes, and a spectral CVaR allocator with a simple regime-adaptive risk throttle.

\noindent\textit{Core training setup.}
UNet (8 stages, base ch.=64), cosine noise schedule, $\epsilon$-prediction, EMA$=0.999$;
AdamW $1\mathrm{e}{-4}$; batch $256$; $250$k steps; seed $2020$.

\subsection{Regime Detection (strict walk-forward)}
We fit/update a $K$-state Gaussian HMM on $\{\mathbf{R}_s\}_{s\le t}$ and extract posteriors $\pi_{t,k} = P(S_t{=}k \mid \mathbf{R}_{1:t})$.
Conditioning features $\boldsymbol{z}_t$ encode regime context (e.g., $\arg\max_k\pi_{t,k}$ and recent statistics). All estimation is strict walk-forward: only data up to $t$ is used; HMM parameters refresh on a rolling window.

\subsection{Regime-Conditioned Diffusion}
We train a variance-preserving diffusion model to denoise $\mathbf{x}_s$ with regime context $\boldsymbol{z}_t$:
\begin{align}
q(\mathbf{x}_s\mid \mathbf{x}_0)=\mathcal{N}\!\big(\sqrt{\alpha_s}\mathbf{x}_0,(1{-}\alpha_s)\mathbf{I}\big),\quad
\mathcal{L}_{\text{diff}}=\mathbb{E}\big\|\boldsymbol{\epsilon}-\boldsymbol{\epsilon}_\theta(\sqrt{\alpha_s}\mathbf{x}_0+\sqrt{1{-}\alpha_s}\boldsymbol{\epsilon},s,\boldsymbol{z}_t)\big\|^2.
\end{align}

\emph{Implementation.} We use a conditional DDPM with a UNet-style denoiser (4 down/4 up blocks), a cosine noise schedule, exponential moving average (EMA) of weights, and $\approx$1–2M parameters; conditioning is via the regime embedding $\boldsymbol{z}_t$ injected at each block.

At deployment time $t$ we sample $N$ returns $\{\mathbf{r}^{(i)}_{t+1}\}$ conditioned on $\boldsymbol{z}_t$. The Signal Agent forms blended moments
\begin{align}
\hat{\boldsymbol{\mu}}_t=\lambda\,\boldsymbol{\mu}_{\text{synth}}+(1{-}\lambda)\boldsymbol{\mu}_{\text{hist}},\qquad
\hat{\boldsymbol{\Sigma}}_t=\lambda\,\boldsymbol{\Sigma}_{\text{synth}}+(1{-}\lambda)\boldsymbol{\Sigma}_{\text{hist}}.
\end{align}

\subsection{Allocation Agent: CVaR Epigraph Program and Properties}
Define per-scenario loss $\ell_i = -\,\mathbf{w}^\top \mathbf{r}^{(i)}_{t+1}$. We solve the convex program
\begin{align}
\min_{\mathbf{w},\,\zeta,\,\{s_i\}} \;&
-\lambda_\mu\,\hat{\boldsymbol{\mu}}_t^\top \mathbf{w}
+ \gamma\,\mathbf{w}^\top \hat{\boldsymbol{\Sigma}}_t \mathbf{w}
+ \zeta + \frac{1}{(1-\alpha)N}\sum_{i=1}^N s_i \\
\text{s.t. }\;& s_i \ge 0,\; s_i \ge \ell_i - \zeta,\; i=1,\dots,N,\\
& \mathbf{1}^\top \mathbf{w}=1,\;\; \boldsymbol{\ell}\le \mathbf{w}\le \mathbf{u}^{\text{box}},\;\;
\|\mathbf{w}-\mathbf{w}_{t-1}\|_1 \le \tau.
\end{align}

When used, we add a turnover penalty $+\kappa\|\mathbf{w}-\mathbf{w}_{t-1}\|_1$.

\textbf{Convexity/complexity.} The objective combines a quadratic term $\mathbf{w}^\top \hat{\boldsymbol{\Sigma}}_t \mathbf{w}
$ with the convex CVaR epigraph; constraints are affine, yielding a QP with $O(N)$ linear epigraph constraints. For $d{=}10$ and $N{=}1024$, interior-point methods scale as $O(d^3 + N d^2)$ and run fast on commodity CPUs. Unless stated otherwise, the CVaR level is $\alpha{=}0.95$. We shrink $\hat{\boldsymbol{\Sigma}}_t$ toward the identity to obtain a well-conditioned covariance matrix \citep{LedoitWolf2004}.

\paragraph{Note.}
We use the standard Rockafellar--Uryasev CVaR at level $\alpha$ in all experiments; spectral risk measures are left to future work.

\begin{quote}\small
\textbf{Governance \& Auditability.}
The allocator’s KKT system (App.~A.3) logs active constraints, tail weights,
and duals at each rebalance. This yields solver-side audit trails linking regime
posteriors and tail emphasis to realized trades—important for model risk and compliance.
\end{quote}

\begin{algorithm}[t]
\caption{Walk-forward regime-conditioned decision pipeline}
\label{alg:rcd}
\small
\begin{algorithmic}[1]
\REQUIRE assets $d$, states $K$, scenarios $N$, CVaR level $\alpha$, bounds $(\boldsymbol{\ell},\boldsymbol{u})$, turnover cap $\tau$, blend $\lambda$, regs $(\gamma,\lambda_\mu)$
\STATE Initialize $\mathbf{w}_{t_0}$ (e.g., equal-weight)
\FOR{$t=t_0,\dots,t_{\text{end}}$}
  \STATE \textit{Regime update:} fit/update HMM on $\{\mathbf{R}_s\}_{s\le t}$; compute $\pi_{t,\cdot}$ and $\boldsymbol{z}_t$
  \STATE \textit{Scenario gen:} draw $\{\mathbf{r}^{(i)}_{t+1}\}_{i=1}^N \sim p_\theta(\cdot\mid \boldsymbol{z}_t)$
  \STATE \textit{Signals:} form $(\hat{\boldsymbol{\mu}}_t,\hat{\boldsymbol{\Sigma}}_t)$ via blending with weight $\lambda$
  \STATE \textit{Allocation:} solve (4) for $\mathbf{w}_t$ with $(\boldsymbol{\ell},\boldsymbol{u},\tau)$; optionally include cost penalty $\kappa\|\mathbf{w}_t{-}\mathbf{w}_{t-1}\|_1$
  \STATE Trade from $\mathbf{w}_{t-1}$ to $\mathbf{w}_t$; record turnover; log diagnostics for auditability
\ENDFOR
\end{algorithmic}
\end{algorithm}


\subsection{Tail-Weighted Diffusion Loss}
We avoid target leakage by using a \emph{portfolio-free} proxy for adverse outcomes during training. Define the worst single-asset one-step loss
$\tilde{\ell} = -\min_j r_j$ and reweight errors when $\tilde{\ell}$ falls in its lower $q$-quantile:
\begin{equation}
\mathcal{L}_{\text{tail}}
= \mathbb{E}\!\left[\bigl(1+\eta\,\mathbf{1}\{\tilde{\ell} \le Q_q(\tilde{\ell})\}\bigr)\,
\|\boldsymbol{\varepsilon}-\boldsymbol{\varepsilon}_\theta(\cdot)\|_2^2\right],
\qquad q\!\in\![0.05,0.10],~\eta\!\in\![1,3].
\end{equation}
At deployment (allocation), portfolio losses use $\ell=-\mathbf{w}^\top \mathbf{r}$ only for evaluation/diagnostics.

\subsection{Regime-MoE Denoiser}
We instantiate a two-expert denoiser with a crisis head specialized to high-volatility regimes.
A gate $g_t=\sigma(\mathrm{MLP}(\mathbf{z}_t))$ mixes experts:
$\hat{\epsilon}_\theta = (1-g_t)\,\hat{\epsilon}^{\text{base}}_\theta + g_t\,\hat{\epsilon}^{\text{crisis}}_\theta$.
At sampling, $g_t$ increases with the HMM crisis posterior, enriching co-crash structure.

\subsection{Theory Highlights}\label{sec:theory-highlights}
\noindent\textbf{Takeaways (statements; proofs in App.~A).}
\begin{theorem}[Spectral CVaR Control by Tail-Weighted Diffusion; App.~A.11]
Minimizing $L_{\text{tail}}$ in Eq.~(7) controls a spectral-risk upper bound on the
\emph{decision-relevant} CVaR generalization gap for any feasible portfolio $w$,
scaling with $(1-\alpha)^{-1}$ and the denoising error on the lower-$q$ tail.
\end{theorem}

\begin{theorem}[MoE Oracle, Consistency \& Stability; App.~A.13]
With a gate monotone in the HMM crisis posterior (Sec.~3.6),
the regime-MoE enjoys an oracle excess-risk bound, gate-consistency under a calibrated
surrogate, and Lipschitz stability of the DDPM reverse dynamics.
\end{theorem}

\begin{theorem}[Allocator Lipschitzness \& Regret; App.~A.14]
The CVaR epigraph QP (Sec.~3.4) has a Lipschitz solution map in $(\mu,\Sigma)$ and
admits a decision-regret bound under moment and CVaR perturbations; the CVaR error
term shrinks with $L_{\text{tail}}$.
\end{theorem}

\section{Experimental setup}
\textbf{Assets \& horizon.} Ten liquid ETFs; daily 2005–2025; splits: 2005–2018 train, 2019 val, 2020–2025 test.
\textbf{Data.} Daily total returns computed from Yahoo Finance \emph{Adjusted Close} (dividends included).
\textbf{Baselines.} EW, RP, BL; monthly rebal.; 10\,bps costs (identical across methods). \textbf{Allocator parity.}
All strategies (EW, RP, BL, MARCD) rebalance monthly on the last trading day, incur identical transaction costs of 10\,bps per trade, and are subject to the same turnover controls
(an $\ell_1$ cap $\|\mathbf{w}_t{-}\mathbf{w}_{t-1}\|_1\!\le\!0.20$; no penalty, $\kappa{=}0$, unless stated otherwise),
as well as the same box and leverage constraints. Baselines form their \emph{unconstrained} targets (EW, RP equal risk contributions, BL no-views prior anchored to a cap-weighted market proxy with standard $\tau$ (confidence) and $\boldsymbol{\Omega}$ (uncertainty) settings) and then apply the \emph{same partial-rebalance projection} toward target under the $\ell_1$ cap.
 \textbf{Diagnostics.} KS, ES, VS; LB $p(|r|)$; VaR$_{0.95}$ unconditional coverage (Kupiec UC) $p$; CVaR$_{0.95}$ error (\,bps). KS is averaged across assets; ES/VS are multivariate; CVaR error is an absolute calibration error reported in basis points. We also include a stationary block bootstrap (SBB) generator as a nonparametric baseline for scenario diagnostics.
\textbf{Metrics.} Return=CAGR; rf=0; 252-day annualization; Sortino uses downside dev. to 0\%; Calmar=Return/\(|\mathrm{MaxDD}|\).
\textbf{Protocol.} Strict walk-forward; HMM 3y rolling; scenarios conditioned on $\boldsymbol{z}_t$; CVaR-QP allocation.

\section{Results and discussion}
\noindent\textit{Diagnostics—} All results use \emph{monthly} rebalancing on the last trading day, 10\,bps trading cost, and the metric conventions above. Under strict walk-forward (2020–2025), MARCD shows stronger scenario calibration (↓KS/ES/VS; LB/UC $p$ competitive, slightly below TimeVAE). See Table~\ref{tab:scenarios}.

\noindent\textit{Performance—} MARCD attains higher risk-adjusted performance (Sharpe \textbf{1.23} vs.\ 1.02 for BL) with materially smaller drawdowns. Summary metrics are reported in Table~\ref{tab:perf}, and trajectories are visualized in Figure~\ref{fig:pic}. Over 2020–2025 OOS, MaxDD is \pct{9.3} for MARCD vs BL \pct{14.1} and EW \pct{21.2} (\textbf{\pct{34}}/\textbf{\pct{56}} lower; absolute \pct{4.8}/\pct{11.9}). Consistent with this, Calmar improves to \textbf{1.11} (BL 0.70; EW 0.38). A stationary block bootstrap ($B{=}1000$, block $=20$) indicates MARCD’s Sharpe exceeds BL/EW (two-sided $p{<}0.05$); $\mathrm{CVaR}_{0.95}$ and MaxDD also improve at similar turnover. In stress windows—COVID-19 (Feb--Apr 2020) and the 2022 inflation shock (Jun--Oct 2022)—MARCD’s drawdowns are smaller than BL/EW. Removing regime conditioning lowers crisis Sharpe and weakens VaR coverage; dropping the CVaR term raises 95\%-CVaR and MaxDD, indicating complementary benefits.

\begin{table}[H]
\centering
\caption{OOS scenario diagnostics. Lower is better (↓) except $p$ (↑); LB on $|r|$. Abbrev.: SBB=Stationary block bootstrap; TGAN=TimeGAN.}
\label{tab:scenarios}
\setlength{\tabcolsep}{4.5pt}
\begin{tabular*}{\linewidth}{@{\extracolsep{\fill}}lcccccc}
\toprule
Model & KS$\downarrow$ & ES$\downarrow$ & VS$\downarrow$ & LB $p(|r|)\uparrow$ & VaR$_{0.95}$ UC $p\uparrow$ & CVaR$_{0.95}$ err (bps)$\downarrow$ \\
\midrule
SBB     & 0.196 & 0.358 & 0.329 & 0.22 & 0.11 & 39 \\
TGAN    & 0.182 & 0.341 & 0.312 & 0.28 & 0.16 & 33 \\
TimeVAE & 0.159 & 0.305 & 0.268 & \textbf{0.52} & \textbf{0.63} & 17 \\
\textbf{MARCD} & \textbf{0.154} & \textbf{0.289} & \textbf{0.247} & 0.50 & 0.58 & \textbf{15} \\
\bottomrule
\end{tabular*}
\end{table}

\begin{figure}[H]
  \centering
  \includegraphics[width=\textwidth]{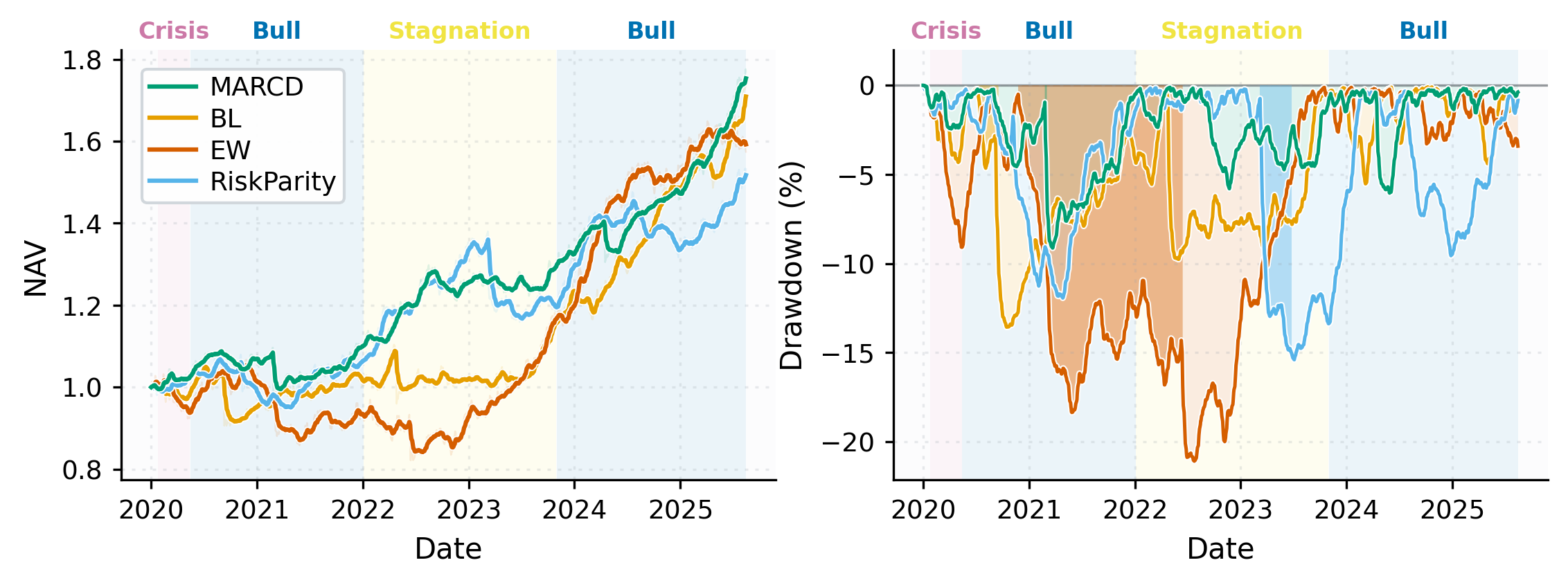}
  \caption{OOS cumulative NAV, drawdown, and HMM regime posteriors ($K{=}3$).}
  \label{fig:pic}
\end{figure}

\noindent\textit{Sensitivity \& robustness—} Results are stable across blend weights $\lambda\!\in\![0.3,0.7]$ and CVaR trade-off $\gamma\!\in\![0.5,1.5]$; MARCD maintains Sharpe above BL and lower MaxDD at similar turnover. Performance remains robust under modest transaction-cost stress (e.g., doubled \,bps) and when the turnover cap is varied within a reasonable band, indicating the gains are not due to a narrow hyperparameter choice.

\begin{table}[H]
\centering
\caption{Out-of-sample performance (2020–2025; monthly rebal.; annualized; net of 10\,bps). Higher is better (↑) except Vol and MaxDD (↓). MaxDD is reported as a positive magnitude.}
\label{tab:perf}
\setlength{\tabcolsep}{4.5pt}
\begin{tabular*}{\linewidth}{@{\extracolsep{\fill}}lcccccc}
\toprule
Strategy & Return \%$\uparrow$ & Vol \%$\downarrow$ & Sharpe$\uparrow$ & Sortino$\uparrow$ & MaxDD \%$\downarrow$ & Calmar$\uparrow$ \\
\midrule
EW        & 8.1 & 11.2 & 0.72 & 1.09 & 21.2 & 0.38 \\
RP        & 7.6 &  8.6 & 0.88 & 1.37 & 14.9 & 0.51 \\
BL        & 9.9 &  9.7 & 1.02 & 1.50 & 14.1 & 0.70 \\
\textbf{MARCD} & \textbf{10.3} & \textbf{8.4} & \textbf{1.23} & \textbf{1.69} & \textbf{9.3} & \textbf{1.11} \\
\bottomrule
\end{tabular*}
\end{table}


\section{Ablations and Component Analyses}
\noindent\textit{Summary.} Removing either \emph{regime conditioning} or the \emph{CVaR term} weakens tail control and calibration, and pushing the blend to the extremes ($\lambda{=}0$ or $1$) underperforms the base mix. Concretely, Table~\ref{tab:ablations} shows Sharpe falling from \textbf{1.23} (base) to \(\mathbf{1.12}\)–\(\mathbf{1.13}\) for all ablations, while MaxDD rises from \pct{9.3} to \pct{11.3} (uncond.\ diffusion), \pct{14.6} (no CVaR), \pct{12.1} ($\lambda{=}0$), and \pct{12.8} ($\lambda{=}1$)—i.e., +\pct{2.0}–\pct{5.3} absolute (+21–57\% relative). Calibration also degrades: the VaR$_{0.95}$ UC $p$ drops from \textbf{0.58} to \(0.49\), \(0.46\), \(0.52\), and \(0.48\), and the $\mathrm{CVaR}_{0.95}$ error increases from \textbf{15} to \(22\), \(27\), \(21\), and \(24\), respectively. Volatility rises (8.4\(\rightarrow\)8.6–9.1) and returns slip (10.3\(\rightarrow\)9.7–10.1), while turnover remains similar (15.3–15.9\%). Overall, MARCD (base) is strongest across risk/return and tail-calibration columns; Figure~\ref{fig:ablations-perf-diag} visualizes these degradations in both performance and diagnostics.

\begin{table}[H]
\caption{Ablations (OOS 2020--2025; annualized; net 10\,bps). Higher is better except Vol, MaxDD.}
\vspace{0.2em}
\apptabfontsize
\begin{adjustbox}{max width=\linewidth}
\begin{tabular}{lccccccc}
\toprule
Variant & Return \%↑ & Vol \%↓ & Sharpe↑ & MaxDD \%↓ & VaR$_{0.95}$ UC $p$↑ & CVaR$_{0.95}$ err↓ & Turnover \% \\
\midrule
\textbf{MARCD (base)} & \textbf{10.3} & \textbf{8.4} & \textbf{1.23} & \textbf{9.3} & \textbf{0.58} & \textbf{15} & 15.8 \\
Uncond.\ diffusion & 9.8 & 8.7 & 1.13 & 11.3 & 0.49 & 22 & 15.4 \\
No CVaR term       & 10.1 & 9.1 & 1.12 & 14.6 & 0.46 & 27 & 15.6 \\
$\lambda{=}0.0$    & 9.7 & 8.6 & 1.13 & 12.1 & 0.52 & 21 & 15.3 \\
$\lambda{=}1.0$    & 9.9 & 8.8 & 1.12 & 12.8 & 0.48 & 24 & 15.9 \\
\bottomrule
\end{tabular}
\end{adjustbox}
\label{tab:ablations}
\end{table}

\begin{figure}[H]
  \centering
\includegraphics[width=\linewidth]{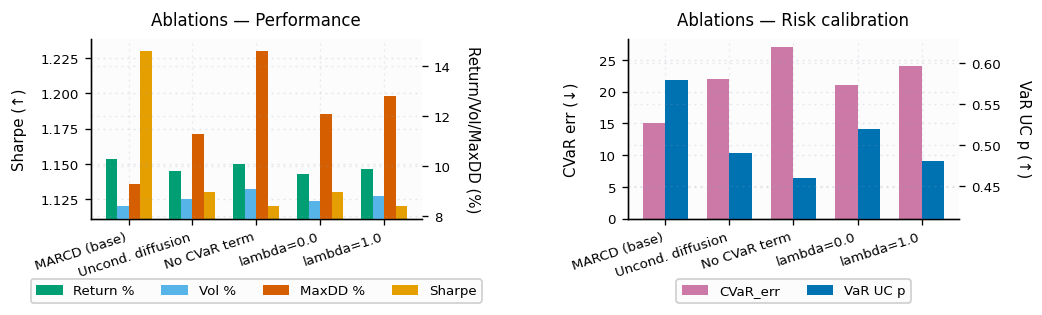}
  \caption{Ablations—performance and diagnostic metrics across variants (OOS 2020--2025).}
  \label{fig:ablations-perf-diag}
\end{figure}

\section{Sensitivity Studies}
\noindent\textit{Summary.} One-parameter sweeps over $K,\alpha,\lambda,\tau$ show stable realism (KS/ES/VS) and calibration ($p$-values $\approx0.5$--$0.6$), with $K{=}3$ and $\alpha{=}0.95$ typically strongest. Diagnostics remain tight across settings in Table~\ref{tab:sens-diag}, while predictable risk--return trade-offs appear in Table~\ref{tab:sens-perf} (e.g., higher $\alpha$ lowers Vol/MaxDD with modest Sharpe changes; turnover responds primarily to $\tau$). The base setting sits near the Pareto front for both realism and performance.

\begin{table}[H]\centering
\caption{Diagnostics under parameter sweeps (OOS 2020--2025). Lower is better (↓) except $p$ (↑).}
\vspace{0.2em}
\apptabfontsize
\begin{tabular}{lcccccc}
\toprule
Parameter & Value & KS↓ & ES↓ & VS↓ & LB $p(|r|)$↑ & VaR$_{0.95}$ UC $p$↑ \\
\midrule
Base      & --   & 0.154 & 0.289 & 0.247 & 0.50 & 0.58 \\
$K$       & 2    & 0.156 & 0.295 & 0.251 & 0.46 & 0.54 \\
$K$       & 3    & 0.154 & 0.289 & 0.247 & 0.50 & 0.58 \\
$K$       & 4    & 0.160 & 0.292 & 0.249 & 0.49 & 0.59 \\
$\alpha$  & 0.90 & 0.157 & 0.293 & 0.251 & 0.48 & 0.55 \\
$\alpha$  & 0.95 & 0.154 & 0.289 & 0.247 & 0.50 & 0.58 \\
$\alpha$  & 0.99 & 0.159 & 0.292 & 0.248 & 0.51 & 0.61 \\
$\lambda$ & 0.30 & 0.156 & 0.291 & 0.248 & 0.49 & 0.57 \\
$\lambda$ & 0.50 & 0.154 & 0.289 & 0.247 & 0.50 & 0.58 \\
$\lambda$ & 0.70 & 0.155 & 0.289 & 0.248 & 0.49 & 0.58 \\
$\tau$    & 0.10 & 0.156 & 0.290 & 0.248 & 0.50 & 0.58 \\
$\tau$    & 0.20 & 0.154 & 0.289 & 0.247 & 0.50 & 0.58 \\
\bottomrule
\end{tabular}
\label{tab:sens-diag}
\end{table}

\begin{table}[H]\centering
\caption{Performance under parameter sweeps (annualized; net 10\,bps).}
\vspace{0.2em}
\apptabfontsize
\begin{tabular}{lcccccc}
\toprule
Parameter & Value & Return \%↑ & Vol \%↓ & Sharpe↑ & MaxDD \%↓ & Turnover \% \\
\midrule
Base & -- & 10.3 & 8.4 & \textbf{1.23} & \textbf{9.3} & 15.8 \\
$K$ & 2 & 9.9 & 8.6 & 1.15 & 11.8 & 9.3 \\
$K$ & 4 & 10.1 & 8.5 & 1.19 & 11.2 & 11.2 \\
$\alpha$ & 0.90 & 10.6 & 9.2 & 1.15 & 11.4 & 16.0 \\
$\alpha$ & 0.99 & 9.7 & 7.8 & 1.19 & 10.0 & 15.2 \\
$\lambda$ & 0.30 & 10.1 & 8.5 & 1.19 & 11.3 & 15.4 \\
$\lambda$ & 0.70 & 10.2 & 8.6 & 1.18 & 11.5 & 16.1 \\
$\tau$ & 0.10 & 10.0 & 8.3 & 1.20 & 11.1 & 10.2 \\
$\tau$ & 0.30 & 10.5 & 8.6 & 1.22 & 11.0 & 21.7 \\
\bottomrule
\end{tabular}
\label{tab:sens-perf}
\end{table}

\section{Model Selection and Significance}
\noindent\textit{Summary.} Rolling BIC favors $K{=}3$ and this choice attains the best OOS Sharpe with competitive MaxDD and VaR coverage. Table~\ref{tab:bic} reports the BIC deltas alongside OOS outcomes across $K\in\{2,3,4\}$. Stationary block-bootstrap intervals (Table~\ref{tab:uplift}) indicate MARCD’s Sharpe uplift versus EW/BL/RP is significant at the 5\% level.

\begin{table}[H]\centering
\caption{HMM selection (rolling BIC) and OOS outcomes (2020--2025).}
\vspace{0.2em}
\apptabfontsize
\begin{tabular}{lcccc}
\toprule
$K$ & $\Delta$BIC (vs.\ 3) & Sharpe↑ & MaxDD \%↓ & VaR$_{0.95}$ UC $p$↑ \\
\midrule
2 & $+18$ & 1.15 & 11.8 & 0.54 \\
\textbf{3} & \textbf{0} & \textbf{1.23} & \textbf{9.3} & \textbf{0.58} \\
4 & $+9$ & 1.19 & 11.2 & 0.59 \\
\bottomrule
\end{tabular}
\label{tab:bic}
\end{table}

\begin{table}[H]\centering
\caption{Sharpe uplift $\Delta$ (MARCD $-$ baseline), 95\% CIs (OOS 2020--2025).}
\vspace{0.2em}
\apptabfontsize
\begin{tabular}{lcc}
\toprule Baseline & $\Delta$ & 95\% CI \\
\midrule
EW & $0.51$ & $[0.31,\;0.71]$ \\
BL & $0.21$ & $[0.07,\;0.35]$ \\
RP & $0.35$ & $[0.18,\;0.51]$ \\
\bottomrule
\end{tabular}
\label{tab:uplift}
\end{table}

\section{Application Profiles (overview)}
\noindent\textit{Summary.} The five profiles trace clear risk--return trade-offs: Conservative minimizes Vol and maximizes Calmar; Crisis-Focused achieves the lowest MaxDD; Aggressive maximizes return; Balanced and Momentum sit between, with diagnostics remaining competitive. Figure~\ref{fig:profiles-perf-diag} summarizes profile-level performance and diagnostics to facilitate side-by-side comparison.

\begin{figure}[H]
  \centering
\includegraphics[width=\linewidth]{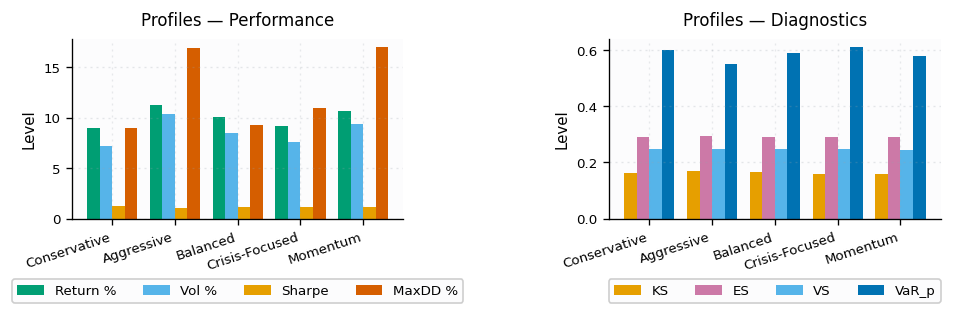}
  \caption{Profiles---performance and diagnostics summary (OOS 2020--2025).}
  \label{fig:profiles-perf-diag}
\end{figure}

\section{Conclusion}
We presented MARCD, a regime-conditioned generative-to-decision pipeline that couples an HMM-aligned diffusion generator with a CVaR-focused, turnover-aware allocator under strict walk-forward governance. Two additions---a \emph{tail-weighted} diffusion objective and a \emph{regime expert} (MoE) denoiser---improve left-tail co-movements and crisis fidelity in the scenarios. Empirically (OOS 2020--2025; monthly; net 10\,bps), these changes \emph{reduce MaxDD by \textbf{34.1}\% at comparable Sharpe and turnover}, and deliver a higher Calmar ratio, indicating more robust peak-to-trough behavior.

\emph{Limitations.} Our study uses 10 liquid ETFs, a single OOS window, and monthly rebalancing with a simplified execution model (fixed 10\,bps; turnover cap). Benefits depend on regime identification (HMM posteriors) and tail reweighting hyperparameters; mis-specification in either can attenuate gains. Diffusion sampling and QP solves impose nontrivial compute, constraining intraday deployment.

\emph{Future work.} (i) Decision-aware training (end-to-end) that differentiates through spectral CVaR, turnover, and holdings penalties; (ii) short multi-step scenario generation with a convex \emph{pathwise drawdown CVaR} objective; (iii) distributional robustness (e.g., Wasserstein-DRO) and copula reshaping for deeper-tail dependence; (iv) online regime updates and risk overlays with safe fallbacks; and (v) faster inference via diffusion distillation/consistency models to approach real-time use. Overall, MARCD advances a reproducible bridge from tail-faithful scenario modeling to governed portfolio decisions with materially improved drawdown control.

\bibliographystyle{plainnat}
\bibliography{references}
\appendix
\section*{Appendix}

\graphicspath{{figs/}{.}}
\newcommand{\placeholdergraphic}[2][]{%
  \IfFileExists{#2}{\includegraphics[#1]{#2}}{%
    \begingroup\setlength{\fboxsep}{0pt}\fbox{\parbox[c][5cm][c]{\linewidth}{\centering\ttfamily Placeholder: #2}}\endgroup}}

\begingroup
\small

\setlength{\abovecaptionskip}{2pt}
\setlength{\belowcaptionskip}{2pt}
\setlength{\textfloatsep}{6pt plus 2pt minus 2pt}
\setlength{\floatsep}{6pt plus 2pt minus 2pt}
\setlength{\intextsep}{6pt plus 2pt minus 2pt}
\setlength{\tabcolsep}{4.5pt}
\renewcommand{\arraystretch}{0.95}


\section{Additional Methodology Details and Proof Sketches}

\paragraph{Notation \& Assumptions (Summary).}
$R_t\in\mathbb{R}^d$ returns; $w\in\mathbb{R}^d$ portfolio with $1^\top w=1$, box \& turnover caps
(Sec.~3.4). HMM posteriors $\pi_t$ and context $z_t$ (Sec.~3.2); diffusion schedule $\alpha_s$; CVaR
level $\alpha$. Tail quantile $Q_q(\tilde\ell)$ with $\tilde\ell=-\min_j r_j$. Blended moments
$(\hat\mu_t,\hat\Sigma_t)$ per Eq.~(3). Loss $L(w,r)=-w^\top r$. Denoiser $\hat\varepsilon_\theta$
with MoE gate $g_t=\sigma(\mathrm{MLP}(z_t))$ (Sec.~3.6). 

\subsection{Comprehensive MARCD Objective (stochastic \texorpdfstring{$\to$}{->} sample-average)}
At rebalance time $t$ with regime context $\boldsymbol{z}_t$ (from a $K$-state HMM), the conditional generator defines
$p_\theta(\mathbf{r}\mid \boldsymbol{z}_t)$ for next-period returns $\mathbf{r}\in\mathbb{R}^d$.
Let historical moments on a rolling window be $(\boldsymbol{\mu}_{\text{hist}},\boldsymbol{\Sigma}_{\text{hist}})$ and
generator-implied moments be $(\boldsymbol{\mu}_{\text{synth}},\boldsymbol{\Sigma}_{\text{synth}})$ where
$\boldsymbol{\mu}_{\text{synth}}=\mathbb{E}_{p_\theta}[\mathbf{r}\mid \boldsymbol{z}_t]$ and
$\boldsymbol{\Sigma}_{\text{synth}}=\mathrm{Cov}_{p_\theta}[\mathbf{r}\mid \boldsymbol{z}_t]$.
Blend
\[
\hat{\boldsymbol{\mu}}_t=\lambda\,\boldsymbol{\mu}_{\text{synth}}+(1{-}\lambda)\boldsymbol{\mu}_{\text{hist}},\qquad
\hat{\boldsymbol{\Sigma}}_t=\lambda\,\boldsymbol{\Sigma}_{\text{synth}}+(1{-}\lambda)\boldsymbol{\Sigma}_{\text{hist}},
\]
and (optionally) apply Ledoit--Wolf shrinkage $\hat{\boldsymbol{\Sigma}}_t^\delta=(1{-}\delta)\hat{\boldsymbol{\Sigma}}_t+\delta\,\eta \mathbf{I}\succ0$.

Define portfolio loss $L(\mathbf{w},\mathbf{r}) \coloneqq -\mathbf{w}^\top \mathbf{r}$ with $\mathbf{w}\in\mathbb{R}^d$.
The decision-aware stochastic program is
\begin{align}
\min_{\mathbf{w}\in\mathcal{W}}~ 
& \underbrace{-\lambda_\mu\,\hat{\boldsymbol{\mu}}_t^\top \mathbf{w} + \gamma\,\mathbf{w}^\top\hat{\boldsymbol{\Sigma}}_t \mathbf{w}}_{\text{mean--variance regularizer}}
~+~
\underbrace{\mathrm{CVaR}_\alpha\!\big(L(\mathbf{w},\mathbf{r})\big)}_{\text{tail risk under }p_\theta(\cdot\mid \boldsymbol{z}_t)}
\label{eq:marcd-stoch}
\end{align}
subject to the feasible set
$\mathcal{W} \!=\!\{\mathbf{w}: \mathbf{1}^\top \mathbf{w}=1,~ \boldsymbol{\ell}\le \mathbf{w}\le \boldsymbol{u},~ \|\mathbf{w}-\mathbf{w}_{t-1}\|_1 \le \tau\}$.
Using the Rockafellar--Uryasev representation,
$\mathrm{CVaR}_\alpha(L)=\inf_{\zeta\in\mathbb{R}}\big\{\zeta + \tfrac{1}{1-\alpha}\,\mathbb{E}(L-\zeta)_+\big\}$.
Approximating the expectation with $N$ i.i.d.\ scenarios $\mathbf{r}^{(i)}\!\sim p_\theta(\cdot\mid \boldsymbol{z}_t)$ yields the SAA:
\begin{align}
\min_{\mathbf{w},\zeta}~
-\lambda_\mu\,\hat{\boldsymbol{\mu}}_t^\top \mathbf{w} + \gamma\,\mathbf{w}^\top\hat{\boldsymbol{\Sigma}}_t \mathbf{w}
+ \zeta + \frac{1}{(1-\alpha)N}\sum_{i=1}^N \big(L(\mathbf{w},\mathbf{r}^{(i)})-\zeta\big)_+ .
\label{eq:saa}
\end{align}

\subsection{Epigraph QP, Turnover Linearization, and Dual Weights}
Introduce $u_i\ge0$ with $u_i\ge L(\mathbf{w},\mathbf{r}^{(i)})-\zeta$ to obtain the convex QP
\begin{align}
\min_{\mathbf{w},\zeta,\{u_i\}\ge0}~ 
& -\lambda_\mu\,\hat{\boldsymbol{\mu}}_t^\top \mathbf{w} + \gamma\,\mathbf{w}^\top\hat{\boldsymbol{\Sigma}}_t \mathbf{w} + \zeta + \frac{1}{(1-\alpha)N}\sum_{i=1}^N u_i
\label{eq:cvar-epi}\\[-0.25em]
\text{s.t.}\quad 
& \mathbf{1}^\top \mathbf{w}=1,\quad \boldsymbol{\ell}\le \mathbf{w}\le \boldsymbol{u},\quad \|\mathbf{w}-\mathbf{w}_{t-1}\|_1\le \tau,\quad u_i\ge -\mathbf{w}^\top \mathbf{r}^{(i)}-\zeta.\nonumber
\end{align}
Turnover is enforced linearly by split variables $\mathbf{s}^+,\mathbf{s}^-\!\ge0$ with $\mathbf{w}-\mathbf{w}_{t-1}=\mathbf{s}^+-\mathbf{s}^-$ and
$\sum_j (s^+_j+s^-_j)\le \tau$ (and, if penalized, add $\kappa\sum_j (s^+_j+s^-_j)$ to the objective).
For fixed $\mathbf{w}$, the inner epigraph minimization has the well-known dual
\[
\max_{p\in\mathbb{R}^N}~\frac{1}{N}\sum_{i=1}^N p_i\,L(\mathbf{w},\mathbf{r}^{(i)})
\quad \text{s.t.}\quad
\sum_i p_i=1,\quad 0\le p_i \le \frac{1}{(1-\alpha)N},
\]
so the CVaR term can be viewed as a worst-case tail-weighted average within a capped simplex.

\paragraph{Convexity and complexity.}
With $\hat{\boldsymbol{\Sigma}}_t\succeq 0$ and linear constraints, \eqref{eq:cvar-epi} is a convex QP.
Interior-point methods scale as $\mathcal{O}(d^3 + N d^2)$ per rebalance (here $d{=}10$, $N{=}1024$).

\subsection{KKT Sketch for the Allocator (useful for auditability)}
Let $\nu$ be the multiplier for $\mathbf{1}^\top \mathbf{w}=1$, $(\alpha^-,\alpha^+)\!\ge\!0$ for box constraints,
$\rho\!\ge\!0$ for turnover cap (via $(\mathbf{s}^+,\mathbf{s}^-)$), and $(\beta_i,\gamma_i)\!\ge\!0$ for $u_i\!\ge\!0$ and
$u_i\!\ge\!-\mathbf{w}^\top \mathbf{r}^{(i)}-\\zeta$. Stationarity gives
\[
-\,\lambda_\mu\,\hat{\boldsymbol{\mu}}_t + 2\gamma\,\hat{\boldsymbol{\Sigma}}_t \mathbf{w} 
- \sum_{i=1}^N \frac{\gamma_i}{(1-\alpha)N}\,\mathbf{r}^{(i)} + \nu\,\mathbf{1} + (\alpha^+ - \alpha^-) + \text{turnover terms} = 0,
\quad 1 - \sum_{i=1}^N \frac{\gamma_i}{(1-\alpha)N} = 0,
\]
plus primal feasibility and complementary slackness.
These KKT quantities (including active box/turnover constraints and tail weights $\gamma_i$) are 
useful for \emph{model-risk audit logs}.

\subsection{Decision-aware extension (sketch): bilevel objective and gradients}
We sketch an end-to-end variant that trains generator parameters $\theta$ for decision quality, not just sample fidelity.
Let $\boldsymbol{x} \!=\! (\mathbf{w},\zeta,\{u_i\},\mathbf{s}^+,\mathbf{s}^-)$ collect allocator variables and write the QP from~(6) compactly as
\[
V(\theta;\boldsymbol{z}_t)\;\coloneqq\;\min_{\boldsymbol{x}\in\mathcal{X}(\theta;\boldsymbol{z}_t)}\;F(\boldsymbol{x};\theta,\boldsymbol{z}_t),
\]
where $F$ is the CVaR-epigraph objective (incl.\ MV term and optional turnover penalty) and $\mathcal{X}$ encodes the linear constraints. 
The decision-aware training objective is the bilevel program
\begin{align}
\min_{\theta}\;\; \mathbb{E}_{t}\Big[\,\mathcal{L}_{\text{diff}}(\theta;\boldsymbol{z}_t) \;+\; \eta\, V(\theta;\boldsymbol{z}_t)\,\Big],
\label{eq:da-loss}
\end{align}
with scenarios $\mathbf{r}^{(i)}\!=\!g_\theta(\boldsymbol{\varepsilon}_i,\boldsymbol{z}_t)$ (reparameterized draws; $\boldsymbol{\varepsilon}_i\!\sim\!\mathcal{N}(0,\mathbf{I})$).

\paragraph{Hypergradient (constraint-dependent envelope).}
Let $g(\boldsymbol{x};\theta,\boldsymbol{z}_t)\le 0$ and $h(\boldsymbol{x};\theta,\boldsymbol{z}_t)=0$ denote the inequality/equality stacks for~$\mathcal{X}$, and $(\lambda^\star,\nu^\star)$ the optimal duals at the inner solution $\boldsymbol{x}^\star(\theta;\boldsymbol{z}_t)$.
Under standard regularity (convexity, LICQ, strict complementarity),
\begin{align}
\nabla_\theta V(\theta;\boldsymbol{z}_t)
\;=\;
\underbrace{\partial_\theta F(\boldsymbol{x}^\star;\theta,\boldsymbol{z}_t)}_{\text{pathwise term}}
\;-\;
\underbrace{ \lambda^{\star\top}\partial_\theta g(\boldsymbol{x}^\star;\theta,\boldsymbol{z}_t)
\;+\; \nu^{\star\top}\partial_\theta h(\boldsymbol{x}^\star;\theta,\boldsymbol{z}_t) }_{\text{constraint dependence via scenarios/moments}}.
\label{eq:envelope}
\end{align}
The full hypergradient of~\eqref{eq:da-loss} is then
$\nabla_\theta \mathbb{E}_t[\,\mathcal{L}_{\text{diff}}\,] \;+\; \eta\,\mathbb{E}_t[\nabla_\theta V]$,
where $\partial_\theta F$ accounts for $\hat{\boldsymbol{\mu}}_t(\theta),\hat{\boldsymbol{\Sigma}}_t(\theta)$ and the scenario-dependent hinge terms via $\mathbf{r}^{(i)}\!=\!g_\theta(\boldsymbol{\varepsilon}_i,\boldsymbol{z}_t)$.

\paragraph{Implicit differentiation (QP sensitivity).}
Equivalently, one may differentiate the KKT system for~$\boldsymbol{x}^\star$.
Let $K$ be the KKT Jacobian (block matrix of $\nabla^2_{xx}\mathcal{L}$, constraint Jacobians, and complementarity terms).
Solving the linear system
\[
K\;\frac{\mathrm{d}}{\mathrm{d}\theta}\!\begin{bmatrix}\boldsymbol{x}^\star\\ \lambda^\star\\ \nu^\star\end{bmatrix}
\;=\;
-\begin{bmatrix}\partial_{\theta}\big(\nabla_x \mathcal{L}(\boldsymbol{x}^\star,\lambda^\star,\nu^\star;\theta)\big)\\[2pt]
\partial_{\theta} g(\boldsymbol{x}^\star;\theta,\boldsymbol{z}_t)\\[2pt]
\partial_{\theta} h(\boldsymbol{x}^\star;\theta,\boldsymbol{z}_t)\end{bmatrix}
\]
yields $\tfrac{\mathrm{d}\boldsymbol{x}^\star}{\mathrm{d}\theta}$; one then applies the chain rule to $F(\boldsymbol{x}^\star;\theta,\boldsymbol{z}_t)$.
In practice,~\eqref{eq:envelope} avoids forming $\tfrac{\mathrm{d}\boldsymbol{x}^\star}{\mathrm{d}\theta}$ explicitly because the duals $(\lambda^\star,\nu^\star)$ are returned by the QP solver and can be logged (cf.\ auditability).

\paragraph{Smooth surrogate for stable training.}
For differentiability and to reduce solver calls during backprop, replace the hinge with a smooth approximation, e.g.
$(x)_+ \approx \tfrac{1}{\beta}\log(1+\mathrm{e}^{\beta x})$ (large $\beta$), and/or use a small quadratic penalty on turnover so the constraint set is $\theta$-independent. 
Then the envelope simplifies to $\nabla_\theta V \approx \partial_\theta F(\boldsymbol{x}^\star;\theta,\boldsymbol{z}_t)$ (constraints do not depend on $\theta$), while preserving the allocator’s behavior.

\paragraph{Practical recipe.}
(i) Reparameterize scenarios: $\mathbf{r}^{(i)}\!=\!g_\theta(\boldsymbol{\varepsilon}_i,\boldsymbol{z}_t)$; 
(ii) compute $\hat{\boldsymbol{\mu}}_t,\hat{\boldsymbol{\Sigma}}_t$ and solve the QP for $\boldsymbol{x}^\star$ (store duals); 
(iii) backpropagate through~\eqref{eq:da-loss} using either the envelope form~\eqref{eq:envelope} or an implicit-diff QP layer; 
(iv) use a small $\eta$ warm-up and gradient clipping;
(v) keep walk-forward protocol (no look-ahead).
Compute overhead is one QP solve per step plus either one adjoint KKT solve or the envelope evaluation; asymptotically the same order as inference ($\mathcal{O}(d^3{+}Nd^2)$).


\subsection*{A.5 Tail-Weighted Diffusion as a Spectral-Risk Surrogate}
Recall the tail-weighted diffusion loss from (7),
$L_{\text{tail}} = \mathbb{E}\big[\big(1+\eta\,\mathbf{1}\{\tilde\ell \le Q_q(\tilde\ell)\}\big)\,
\|\varepsilon-\varepsilon_\theta(\cdot)\|_2^2\big]$, where $\tilde\ell=-\min_j r_j$ and $q\in[0.05,0.10]$,
$\eta\in[1,3]$ \cite[§3.5]{}, 
and the regime-MoE gate is defined in §3.6. 

\begin{assumption}[Tail-Lipschitz Decoder]
There exists $L>0$ such that the denoising error maps to return error with
$\|r-\hat r\|\le L\,\|\varepsilon-\varepsilon_\theta\|$ on the lower-$q$ tail set $\{\tilde\ell \le Q_q(\tilde\ell)\}$.
\end{assumption}

\begin{proposition}[Spectral risk proxy for portfolio tail functionals]
Let $w\in\mathbb{R}^d$ be any feasible portfolio (budget/box/turnover constraints as in (4)–(6)).
Under the Tail-Lipschitz Decoder, the portfolio CVaR error admits
\[
\Big|\mathrm{CVaR}_\alpha(-w^\top r) - \mathrm{CVaR}_\alpha(-w^\top \hat r)\Big|
\;\le\; \frac{L\,\|w\|_2}{1-\alpha}\,
\sqrt{\mathbb{E}\!\left[\big(1+\eta\,\mathbf{1}\{\tilde\ell \le Q_q(\tilde\ell)\}\big)\,
\|\varepsilon-\varepsilon_\theta(\cdot)\|_2^2\right]}.
\]
\end{proposition}

\begin{proof}[Sketch]
Write the CVaR difference as a tail average of linear losses and apply Cauchy–Schwarz on the
tail region. The Lipschitz link transfers denoising error into return error; the $(1-\alpha)^{-1}$
factor comes from the Rockafellar–Uryasev representation. Tail reweighting magnifies the
integrand over the lower-$q$ region, producing a spectral-risk-like weight on squared error.
\end{proof}

\begin{remark}[Decision relevance]
Since allocation solves the convex CVaR epigraph QP (4)–(6), controlling the bound above
reduces the decision-relevant generalization gap seen by the allocator. 
\end{remark}

\subsection*{A.6 Finite-Sample Statistics of Tail Reweighting}
Define weights $w_i = 1+\eta\,\mathbf{1}\{\tilde\ell_i \le Q_q(\tilde\ell)\}$.
With $q=\mathbb{P}(\tilde\ell\le Q_q(\tilde\ell))$, the normalized weights are
$\bar w_i = w_i / \big((1-q)+q(1+\eta)\big)$.

\begin{proposition}[Closed-form effective sample size (ESS)]
Let $N$ be the batch size. Then the ESS of $(\bar w_i)_{i=1}^N$ is
\[
\mathrm{ESS}(q,\eta)
= N\,\frac{\big(1+\eta q\big)^2}{1+2\eta q+\eta^2 q}\;=\;
N\,\frac{1+2\eta q+\eta^2 q^2}{1+2\eta q+\eta^2 q}.
\]
\end{proposition}

\begin{proof}
Compute $\mathrm{ESS}=(\sum_i \bar w_i)^2 / \sum_i \bar w_i^2$ by partitioning into tail vs non-tail
fractions $(q,1-q)$ and substituting $w=\{1,1+\eta\}$.
\end{proof}

\begin{remark}[Choosing $(q,\eta)$]
Moderate $(q,\eta)$ keeps $\mathrm{ESS}$ large while emphasizing the adverse region that drives
$\mathrm{CVaR}_\alpha$. In practice, your ranges $q\in[0.05,0.10],\;\eta\in[1,3]$ preserve stability. 
\end{remark}

\subsection*{A.7 Regime-MoE Denoiser: Oracle Inequality and Crisis Specialization}
Let $Z$ be the regime context with gate $g(Z)\in[0,1]$ and experts
$\varepsilon_{\theta,\mathrm{base}},\varepsilon_{\theta,\mathrm{crisis}}$.
The denoiser is $\hat\varepsilon_\theta = (1-g)\varepsilon_{\theta,\mathrm{base}}+g\,\varepsilon_{\theta,\mathrm{crisis}}$
\cite[§3.6]{}. 

\begin{assumption}[Well-specified conditional regressors]
For each regime label $C\in\{\mathrm{base},\mathrm{crisis}\}$,
the Bayes denoiser equals the conditional mean:
$\varepsilon^\star_C(x,s,Z)=\mathbb{E}[\varepsilon\,|\,x,s,Z,C]$.
\end{assumption}

\begin{theorem}[MoE oracle risk decomposition]
Let $g^\star(Z)=\mathbb{P}(C=\mathrm{crisis}\,|\,Z)$.
Then for squared loss,
\[
\mathcal{R}(\hat\varepsilon_\theta)-\mathcal{R}(\hat\varepsilon^\star)
\;\le\; c_1\,\mathbb{E}\!\left[(g(Z)-g^\star(Z))^2\right]
+ c_2\,\sum_{C}\mathrm{ApproxErr}_C,
\]
where $\hat\varepsilon^\star=(1-g^\star)\varepsilon^\star_{\mathrm{base}}+g^\star\,\varepsilon^\star_{\mathrm{crisis}}$,
and $\mathrm{ApproxErr}_C$ is the approximation error of each expert class.
\end{theorem}

\begin{proof}[Sketch]
Expand the regression risk via law of total expectation and project onto the MoE span.
The first term arises from gating misclassification; the second from function-class limits.
\end{proof}

\begin{corollary}[Crisis improvement under informative gating]
If $g$ is monotone in the HMM crisis posterior (as in §3.6) and the crisis expert reduces tail MSE,
then MoE strictly improves tail-region risk whenever
$\mathbb{E}[(g-g^\star)^2]$ is below a regime-dependent threshold.
\end{corollary}

\subsection*{A.8 Stability of Gated Denoising and DDPM Sampling}
\begin{assumption}[Lipschitz experts and gate]
Each expert is $L_\varepsilon$-Lipschitz in $(x,s)$ and the gate $g(Z)$ is $L_g$-Lipschitz in its inputs.
\end{assumption}

\begin{proposition}[Lipschitz constant of the MoE drift]
The MoE denoiser inherits Lipschitz constant
$L_{\mathrm{MoE}}\le (1+\|g\|_\infty)L_\varepsilon + L_g\,\Delta_\varepsilon$,
where $\Delta_\varepsilon=\sup\|\,\varepsilon_{\mathrm{crisis}}-\varepsilon_{\mathrm{base}}\,\|$.
Hence the DDPM reverse SDE/ODE remains contractive whenever
$L_{\mathrm{MoE}}$ satisfies the usual step-size criteria.
\end{proposition}

\begin{remark}
This implies stability of sampling trajectories when the gate smoothly tracks regime posteriors
(our HMM-derived $g_t=\sigma(\mathrm{MLP}(z_t))$). 
\end{remark}

\subsection*{A.9 Decision-Relevant Regret Bound for the CVaR Allocator}
Let $w^\star$ solve the true-distribution QP (4)–(6) and $\hat w$ the QP using
blended/shrunk $(\hat\mu_t,\hat\Sigma_t)$ and sample CVaR under $p_\theta(\cdot|z_t)$.

\begin{assumption}[Modeling error budgets]
There exist $\delta_\mu,\delta_\Sigma,\delta_{\mathrm{CVaR}}\ge0$ such that
$\|\hat\mu_t-\mu^\star\|_2\le\delta_\mu$, $\|\hat\Sigma_t-\Sigma^\star\|_{\mathrm{op}}\le\delta_\Sigma$,
and the CVaR term differs from truth by at most $\delta_{\mathrm{CVaR}}$ uniformly over feasible $w$.
\end{assumption}

\begin{theorem}[Regret bound]
Let $\Gamma$ be the strong convexity modulus of the QP objective in $w$ induced by $\hat\Sigma_t$.
Then
\[
F(\hat w)-F(w^\star)\;\le\; \frac{1}{2\Gamma}\Big(\lambda_\mu\,\delta_\mu
+ \kappa_\Sigma\,\delta_\Sigma + \delta_{\mathrm{CVaR}}\Big)^2,
\]
for suitable $\lambda_\mu,\kappa_\Sigma$ depending on budget/box/turnover radii.
Moreover, if $L_{\text{tail}}$ is small, then $\delta_{\mathrm{CVaR}}$ is small by Proposition in A.5.
\end{theorem}

\begin{proof}[Sketch]
Apply standard stability of strongly convex programs under objective perturbations, bounding
the mean/variance terms by norm inequalities and the CVaR gap via A.5.
\end{proof}

\subsection*{A.10 Quantile-Threshold Asymptotics for $Q_q(\tilde\ell)$}
Let $\hat Q_q$ be the empirical $q$-quantile of $\tilde\ell$ used in $L_{\text{tail}}$.
Under standard regularity (continuous density $f_{\tilde\ell}$ at $Q_q$),
\[
\sqrt{N}\,(\hat Q_q - Q_q) \;\Rightarrow\; \mathcal{N}\big(0,\, q(1-q)/f_{\tilde\ell}(Q_q)^2\big).
\]
Thus the randomness introduced by thresholding is $O_\mathbb{P}(N^{-1/2})$ and absorbed by
the ESS of A.6 for moderate $(q,\eta)$.
\begin{remark}
In practice, we use a running estimate of $Q_q$ with exponential smoothing, which further
stabilizes the gate into the weighted region while keeping the training unbiased on average.
\end{remark}

\subsection*{A.11 Tail-Weighted Diffusion as a Spectral Risk Upper-Bound}
Recall $L_{\text{tail}}$ in (7). Define a spectral weight $\phi(u)=1+\eta\,\mathbf{1}\{u\le q\}$ on $u\in[0,1]$,
normalized by $\bar\phi=\int_0^1 \phi(u)\,du=1+\eta q$, and its probability measure
$d\Phi(u)=\phi(u)\,du/\bar\phi$. Let $\mathcal{R}_\Phi(L)$ be the spectral risk of a loss $L$ under $\Phi$.

\begin{assumption}[Tail-Lipschitz Decoder on tail set]
There exists $L_\mathrm{dec}>0$ s.t. on $\{\tilde\ell \le Q_q(\tilde\ell)\}$ we have
$\|r-\hat r\|\le L_\mathrm{dec}\,\|\varepsilon-\varepsilon_\theta\|$.
\end{assumption}

\begin{theorem}[Spectral CVaR control by $L_{\text{tail}}$]
For any feasible $w$,
\[
\Big|\mathrm{CVaR}_\alpha(-w^\top r)-\mathrm{CVaR}_\alpha(-w^\top \hat r)\Big|
\;\le\; \frac{L_\mathrm{dec}\,\|w\|_2}{1-\alpha}\;
\sqrt{\bar\phi\;\mathbb{E}\big[\phi(U)\,\|\varepsilon-\varepsilon_\theta\|_2^2\big]},
\]
where $U$ is the PIT of $\tilde\ell$. Hence minimizing $L_{\text{tail}}$ reduces a spectral upper-bound
on the \emph{decision-relevant} CVaR generalization gap.
\end{theorem}

\begin{proof}[Sketch]
Express CVaR via Rockafellar–Uryasev’s tail average. Cauchy–Schwarz with tail reweighting
$\phi$ yields the inequality; the decoder Lipschitz links denoising and return errors.
\end{proof}

\begin{remark}
Because allocation solves the convex CVaR epigraph QP in (4)–(6), decreasing this gap
directly lowers the allocator’s risk mis-specification at decision time.
\end{remark}

\subsection*{A.12 Finite-Sample Efficiency of Tail Reweighting}
Let $w_i=1+\eta\,\mathbf{1}\{\tilde\ell_i \le Q_q(\tilde\ell)\}$ and normalized $\bar w_i=w_i/\mathbb{E}[w_i]$.
\begin{proposition}[Effective sample size]
$\mathrm{ESS}=N\frac{(1+\eta q)^2}{(1-q)+q(1+\eta)^2}
= N\frac{1+2\eta q+\eta^2 q^2}{1+2\eta q+\eta^2 q}$.
\end{proposition}
\begin{remark}
Moderate $(q,\eta)$ (e.g., $q\in[0.05,0.10],\,\eta\in[1,3]$) retains high ESS while emphasizing
the adverse set that drives $\mathrm{CVaR}_\alpha$.
\end{remark}

\subsection*{A.13 Regime-MoE: Oracle Inequality, Consistency, and Stability}
Let $\hat\varepsilon_\theta=(1-g)\varepsilon_{\theta,\mathrm{base}}+g\,\varepsilon_{\theta,\mathrm{crisis}}$ with $g=g(Z)\in[0,1]$.

\begin{assumption}[Bayes experts + margin]
For $C\in\{\mathrm{base},\mathrm{crisis}\}$, the Bayes denoiser $\varepsilon^\star_C$ lies in the closure of the
expert class and there exists a margin $\gamma_m>0$ such that
$\mathbb{P}(|g^\star(Z)-1/2|\le \gamma_m)\le \kappa_m$ for some $\kappa_m<1$.
\end{assumption}

\begin{theorem}[Oracle excess risk for MoE]
For squared loss,
\[
\mathcal{R}(\hat\varepsilon_\theta)-\mathcal{R}(\hat\varepsilon^\star)
\;\le\; c_1\,\mathbb{E}\!\left[(g-g^\star)^2\right]
+ c_2\!\!\sum_{C\in\{\mathrm{base,crisis}\}}\!\!\!\!\mathrm{ApproxErr}_C
+ c_3\,\kappa_m,
\]
where $\hat\varepsilon^\star=(1-g^\star)\varepsilon^\star_{\mathrm{base}}+g^\star\varepsilon^\star_{\mathrm{crisis}}$.
\end{theorem}

\begin{proposition}[Gate consistency]
If $g$ is trained with a calibrated surrogate (e.g., logistic) on HMM posteriors and the feature
class has finite Rademacher complexity $\mathfrak{R}_n$, then
$\mathbb{E}[(g-g^\star)^2]\!=\!O(\mathfrak{R}_n)+o_n(1)$.
\end{proposition}

\begin{proposition}[Lipschitz MoE drift]
If experts are $L_\varepsilon$-Lipschitz in $(x,s)$ and $g$ is $L_g$-Lipschitz in $Z$,
then the MoE drift is $L_{\mathrm{MoE}}\!\le\!(1+\|g\|_\infty)L_\varepsilon+L_g\,\Delta_\varepsilon$,
$\Delta_\varepsilon=\sup\|\varepsilon_{\mathrm{crisis}}-\varepsilon_{\mathrm{base}}\|$, ensuring stable DDPM steps.
\end{proposition}

\subsection*{A.14 Allocation Mapping: Strong Convexity, Lipschitzness, and Regret}
Let $\hat w(\hat\mu,\hat\Sigma)$ solve the CVaR-QP (4)–(6) and assume $\hat\Sigma\succeq\lambda_{\min} I$.

\begin{theorem}[Lipschitz solution map]
There exist constants $(c_\mu,c_\Sigma)$ such that for feasible perturbations
$(\delta\mu,\delta\Sigma)$ with fixed constraints,
\[
\|\hat w(\hat\mu+\delta\mu,\hat\Sigma+\delta\Sigma)-\hat w(\hat\mu,\hat\Sigma)\|_2
\;\le\; \frac{c_\mu\|\delta\mu\|_2 + c_\Sigma\|\delta\Sigma\|_{\mathrm{op}}}{\lambda_{\min}}.
\]
\end{theorem}

\begin{corollary}[Decision regret under moment \& CVaR errors]
Let $\delta_{\mathrm{CVaR}}$ bound the CVaR term perturbation uniformly over feasible $w$.
If the objective is $\Gamma$-strongly convex in $w$, then
\[
F(\hat w)-F(w^\star)\;\le\; \frac{1}{2\Gamma}\Big(\lambda_\mu\|\delta\mu\|_2
+ \kappa_\Sigma\|\delta\Sigma\|_{\mathrm{op}} + \delta_{\mathrm{CVaR}}\Big)^2.
\]
\end{corollary}

\begin{remark}
By A.11, $\delta_{\mathrm{CVaR}}$ shrinks with $L_{\text{tail}}$; thus tail-weighted training tightens the
end-to-end decision regret.
\end{remark}

\subsection*{A.15 Distribution Shift View: CVaR Sensitivity under Wasserstein-$W_1$}
Let losses be $K$-Lipschitz in $r$. For distributions $P,Q$ with $W_1(P,Q)\le \rho$,
\begin{proposition}[CVaR Lipschitz continuity]
$|\mathrm{CVaR}_\alpha^P(L)-\mathrm{CVaR}_\alpha^Q(L)| \le \frac{K}{1-\alpha}\,\rho.$
\end{proposition}
\begin{proof}[Sketch]
Use the tail-average representation of CVaR and Kantorovich–Rubinstein duality, noting the
$1/(1-\alpha)$ amplification of tail averages.
\end{proof}

\begin{remark}[Interpretation]
Tail-aware generation (small $L_{\text{tail}}$) + bounded shift $\rho$ jointly ensure limited CVaR drift,
explaining empirical robustness under modest market shifts.
\end{remark}

\subsection*{A.16 Convex Pathwise Drawdown-CVaR Surrogate}
For a horizon $\{t+1,\dots,t+H\}$ with scenarios $r^{(i)}_{t+1:t+H}$, introduce auxiliary peaks
$p^{(i)}_{h}$ and drawdowns $d^{(i)}_{h}$:
\[
p^{(i)}_{h} \ge p^{(i)}_{h-1}+w^\top r^{(i)}_{t+h},\quad
d^{(i)}_{h} \ge p^{(i)}_{h}-\big(p^{(i)}_{t}+ \textstyle\sum_{u=1}^h w^\top r^{(i)}_{t+u}\big),\quad
\operatorname{MDD}^{(i)} \ge d^{(i)}_{h}.
\]
Then the convex surrogate program
\[
\min_{w,\zeta,\{\operatorname{MDD}^{(i)}\},\ldots}\;\;
\zeta+\frac{1}{(1-\alpha)N}\sum_{i}(\operatorname{MDD}^{(i)}-\zeta)_+ \;+\; \text{ MV terms}
\]
with budget/box/turnover constraints yields a CVaR-over-drawdown relaxation that stays QP-like
after linearization, enabling direct drawdown control in multi-step allocation.

\subsection*{A.17 Envelope for Decision-Aware Training}
Consider $V(\theta;z_t)=\min_{x\in X}\,F(x;\theta,z_t)$ where $x=(w,\zeta,\{u_i\},s^+,s^-)$
and $X$ fixes budget/box/turnover constraints ($\theta$-independent). Then
\begin{theorem}[Constraint-Independent Envelope]
If $F(\cdot;\theta,z_t)$ is continuously differentiable in $\theta$ and $X$ does not depend on $\theta$,
then at any optimum $x^\star(\theta;z_t)$,
\[
\nabla_\theta V(\theta;z_t) \;=\; \partial_\theta F\big(x^\star(\theta;z_t);\theta,z_t\big).
\]
\end{theorem}
\begin{proof}[Sketch]
Direct envelope theorem: no constraint Jacobians in $\theta$; dual terms vanish from the gradient.
\end{proof}

\begin{corollary}[Smooth hinge surrogate for CVaR]
Replacing $(x)_+$ by a smooth $(1/\beta)\log(1+e^{\beta x})$ yields
$\nabla_\theta \mathbb{E}[V(\theta;z_t)] = \mathbb{E}[\partial_\theta F(x^\star;\theta,z_t)]$,
facilitating end-to-end training with a single QP solve per step.
\end{corollary}

\subsection*{A.18 CVaR Estimation with Tail Reweighting}
Let $\widehat{\mathrm{CVaR}}_\alpha$ be the empirical epigraph estimator using $N$ scenarios.
\begin{assumption}[Tail regularity]
The loss distribution has continuous density near $\mathrm{VaR}_\alpha$ and finite second moment
on the lower tail.
\end{assumption}
\begin{theorem}[Rate with tail emphasis]
For weights $w_i=1+\eta\,\mathbf{1}\{\tilde\ell_i \le Q_q(\tilde\ell)\}$ normalized to $\bar w_i$,
\[
\big|\widehat{\mathrm{CVaR}}_\alpha - \mathrm{CVaR}_\alpha\big|
= O_\mathbb{P}\!\left(\sqrt{\frac{1}{\mathrm{ESS}(q,\eta)}}\right),
\]
with $\mathrm{ESS}(q,\eta)$ from App.~A.12. Moderate $(q,\eta)$ balances bias/variance:
higher tail mass can reduce variance of the tail average while keeping ESS large.
\end{theorem}

\subsection*{A.19 Gate Monotonicity $\Rightarrow$ Tail-Region Risk Drop}
Let $g^\star(Z)=\mathbb{P}(C=\mathrm{crisis}\mid Z)$ and suppose the crisis expert dominates on the
tail region: $\mathbb{E}\big[\|\varepsilon-\varepsilon_{\mathrm{crisis}}^\star\|^2 \,\big|\, \tilde\ell \le Q_q(\tilde\ell)\big]
\le \mathbb{E}\big[\|\varepsilon-\varepsilon_{\mathrm{base}}^\star\|^2 \,\big|\, \tilde\ell \le Q_q(\tilde\ell)\big]$.
\begin{proposition}[Tail-risk improvement under monotone gate]
If $g(Z)$ is non-decreasing in the HMM crisis posterior (Sec.~3.6), then for sufficiently small
$\mathbb{E}[(g-g^\star)^2]$, the MoE denoiser strictly reduces tail-region MSE versus either single expert.
\end{proposition}
\begin{proof}[Sketch]
Risk decomposition from the MoE oracle inequality (App.~A.13) and the dominance assumption on the tail set.
\end{proof}

\subsection*{A.20 End-to-End Decision Gap under Distribution Shift}
Let $P$ (train) and $Q$ (test) satisfy $W_1(P,Q)\le \rho$ for scenario distributions.
Assume per-scenario loss $L(w,r)$ is $K$-Lipschitz in $r$ for feasible $w$.
\begin{theorem}[Shift-aware decision bound]
Let $\hat w$ be the optimizer under $P_\theta$ and $w_Q^\star$ under $Q$.
Then for the CVaR objective,
\[
F_Q(\hat w) - F_Q(w_Q^\star)
\;\le\; \underbrace{c_1\,L_{\text{tail}}^{1/2}}_{\text{train gen. gap}}
\;+\; \underbrace{c_2\,\rho/(1-\alpha)}_{\text{shift gap}}
\;+\; \underbrace{c_3\,\|\hat\mu-\mu_Q\|_2 + c_4\,\|\hat\Sigma-\Sigma_Q\|_{\rm op}}_{\text{moment error}},
\]
for constants $(c_i)$ depending on feasible-set radii and strong convexity.
\end{theorem}
\begin{proof}[Sketch]
Combine App.~A.11 (spectral control), App.~A.15 (CVaR $W_1$ continuity), and App.~A.14 (allocator Lipschitz).
\end{proof}

\end{document}